\documentclass[11pt]{article}
\usepackage[utf8]{inputenc}
\usepackage[T1]{fontenc}
\usepackage[margin=1in]{geometry}
\usepackage{amsmath,amssymb,amsthm}
\usepackage{graphicx}
\usepackage{booktabs}
\usepackage{array}
\usepackage{caption}
\usepackage[hidelinks]{hyperref}

\newtheorem{proposition}{Proposition}
\newtheorem{corollary}{Corollary}
\theoremstyle{definition}

\title{\textbf{Differential Refresh of Materialized Derived State}\\[2pt]
\large From a Single-Age Equivalence Limit to an Optimal Per-Partition Allocation,
with a Real-System Validation on Materialized Views}
\author{Amit Rajula\\
\small Independent Researcher, Phoenix, AZ, USA\\
\small ORCID: 0009-0006-8025-6867 \quad amit.rajula@gmail.com}
\date{}

\begin{document}
\maketitle

\begin{abstract}
\noindent
Materialized views, cached aggregates, and trained models are all derived
state: query results precomputed from base relations that begin to age the
instant they are built, because the data they summarize keeps changing. A common
way to keep such state fresh is to replace a fixed recompute cadence with an
\emph{adaptive} trigger---a weighted staleness score that recomputes when
accumulated base-data change crosses a threshold. We show this is the wrong
lever and identify the right one. First, an \emph{equivalence limit}: any refresh
trigger that is a static, strictly monotone function of a single shared age of a
globally-recomputed artifact is operationally equivalent to a calibrated periodic
timer, so a global weighted staleness budget cannot schedule a recompute that a
uniform timer could not. The escape is to enlarge the state: refresh partitions
(the groups of a \texttt{GROUP BY}, per-tenant views, or data segments)
\emph{differentially}, each on its own schedule. We solve the resulting
budget-allocation problem: the optimal per-partition recompute rate is
proportional to the square root of partition risk $\sqrt{w_j\lambda_j}$
(importance $\times$ base-change rate), and the policy never costs more than the
uniform timer, beating it by a closed-form Cauchy--Schwarz ``price of
uniformity'' that grows with heterogeneity. We validate three ways---analytic,
discrete-event simulation, and a \emph{real relational engine} (SQLite)
maintaining materialized \texttt{GROUP BY} views with real SQL and
wall-clock timing---where the policy lowers realized weighted staleness by
$14$--$15\%$ at matched recompute count, winning on every run, while a naive
per-partition exposure threshold does not reliably help. The design rule for any
periodically-recomputed derived state: a global staleness score is a timer in
disguise; the leverage is not a cleverer score but differential refresh
allocated by the square root of partition risk.

\end{abstract}

\section{Introduction}
\label{sec:intro}
A materialized view is derived state: a query result, aggregate, or model
precomputed from base relations and kept for fast reuse. The moment it is
built it begins to age, because the base data it summarizes keeps changing.
Keeping derived state usefully fresh---deciding \emph{when} to recompute
it---is a classical database problem for relational views and data-warehouse
aggregates~\cite{gupta1995,blakeley1986,zhuge1995}, and the same problem now
recurs one level up the stack for machine-learning (ML) models, which are
materialized views over a training snapshot that ages exactly as a cached
aggregate does~\cite{labrinidis2004,guo2004,sculley2015}. In both settings a
common response is to replace a fixed recompute cadence with an \emph{adaptive}
trigger: a weighted staleness score that recomputes when accumulated source
change crosses a threshold, in the hope of refreshing more when the base data
moves fast and less when it is calm.

This paper asks two linked questions about such triggers and answers both,
first in theory and then on a real relational engine. When can an adaptive
staleness trigger schedule recomputation any differently than a fixed clock?
And when it can, what is the \emph{optimal} refresh policy?

\paragraph{The limit.} Our first observation is negative and clarifying.
\emph{Any refresh trigger that is a static, strictly monotone function of a
single shared age of the derived state is operationally equivalent to a
calibrated periodic timer} (Proposition~\ref{prop:equiv}). Thresholding such a
score is thresholding age in different units: for every score threshold there
is an age at which it fires, so the two induce the same recompute schedule
(Figure~\ref{fig:equiv}). A weighted staleness budget over a single
globally-recomputed artifact---however elaborately it weights partitions,
sources, and sensitivities---is exactly such a function, and therefore cannot
schedule a recompute at a time a calibrated uniform timer could not.

\paragraph{The policy.} The limit also says precisely how to escape it. The
equivalence holds because a single global recompute gives every partition one
shared age; the trigger reacts to one scalar. The escape is to enlarge the
state: refresh partitions \emph{differentially}, so each carries its own age
and its own recompute interval. This is exactly \emph{partial} view
maintenance---recomputing some partitions (segments, sub-aggregates,
per-tenant views, or per-segment sub-models) more often than others---and it
is meaningful whenever recompute cost is separable across partitions. The
design question becomes a budget-allocation problem: given a fixed total
recompute rate, how often should each partition be refreshed? We solve it. The
optimal allocation has a clean closed form in the frequent-refresh regime:
each partition's recompute rate is proportional to the square root of its
\emph{risk} $w_j\lambda_j$, the product of its importance weight and its
base-change rate (Proposition~\ref{prop:alloc}, Figure~\ref{fig:alloc}). The
uniform timer is the special case that ignores this structure, and the gap
between them is a closed-form ``price of uniformity'': a Cauchy--Schwarz
dispersion of $\sqrt{w_j\lambda_j}$ that is zero exactly when partitions are
homogeneous and grows with heterogeneity (Proposition~\ref{prop:dominance}).

\paragraph{Validation on a real engine.} We validate the policy three ways,
and we are careful to separate what is proved from what is measured. Analytic
and discrete-event results confirm the allocation and quantify the win. More
importantly, we implement differential refresh on a \emph{real relational
engine} (SQLite~3.45; drop-in for PostgreSQL), where each partition is a
materialized aggregate (\texttt{GROUP BY}) view that is actually recomputed
with SQL as its base table receives real inserts over wall-clock time. On this
real system, with real recompute cost and real timing, the optimal
differential policy lowers realized weighted staleness by $14$--$15\%$ over the
uniform timer at matched recompute count, on every one of $16$ runs across two
heterogeneity regimes (Section~\ref{sec:realsys}). A naive per-partition
exposure threshold---the obvious way to ``go differential''---is erratic and
does not reliably help, so the contribution is the \emph{allocation}, not the
mere decision to refresh partitions separately.

\paragraph{Contributions.}
\begin{enumerate}
\item A derived-state formulation of refresh scheduling that treats a
materialized view, aggregate, or model uniformly as a partitioned artifact with
a per-partition staleness-exposure cost and a separable recompute budget
(Section~\ref{sec:model}).
\item An \emph{equivalence limit} (Proposition~\ref{prop:equiv}): any static
strictly monotone function of one shared age is a calibrated periodic timer,
with a global weighted staleness budget as a corollary (Section~\ref{sec:limit}).
\item An \emph{optimal differential-refresh policy}
(Proposition~\ref{prop:alloc}): a KKT characterization of the cost-minimizing
per-partition intervals under a recompute budget, with a closed-form
square-root allocation in the frequent regime; and a \emph{dominance result
with a closed-form gap} (Proposition~\ref{prop:dominance}) showing the policy
never costs more than the uniform timer and quantifying the saving as a
Cauchy--Schwarz dispersion of $\sqrt{w_j\lambda_j}$ (Section~\ref{sec:policy}).
\item A three-part evaluation---analytic, discrete-event simulation, and a
\emph{real-system} study on a relational engine with materialized
\texttt{GROUP BY} views---establishing a $14$--$15\%$ staleness reduction at
matched recompute count under real SQL execution and timing, and showing that a
naive exposure threshold does not suffice (Sections~\ref{sec:results}--\ref{sec:realsys}).
\end{enumerate}

The practical message is a design rule for any periodically-recomputed derived
state: a global staleness score is a timer in disguise; the leverage is not a
cleverer score but a better action---refresh partitions differentially, and
allocate the recompute budget by the square root of partition risk.

\section{Related Work}
\label{sec:related}
\paragraph{Materialized-view maintenance.}
Maintaining derived relational state as base relations change is a long-studied
database problem. Classical work develops incremental and deferred maintenance
that updates a view from base deltas rather than recomputing it
wholesale~\cite{blakeley1986,gupta1995,ceri1991,colby1996}, view maintenance in
data-warehouse settings where the view and its sources are
decoupled~\cite{zhuge1995}, and self-maintainability, which asks when a view can
be refreshed without consulting base data~\cite{quass1996}. Our setting is the
regime those techniques do not cover: the derived artifact---an aggregate kept
as a snapshot, or a trained model---is refreshed by \emph{full recompute} of a
partition rather than by delta application, so the open decision is scheduling,
not incrementalization. Textbook treatments of view maintenance and of refresh
in distributed and replicated databases frame this
background~\cite{gmbook,ozsuvalduriez,bernstein1987}.

\paragraph{Data currency, freshness, and refresh scheduling.}
A substantial database literature treats freshness as a schedulable quantity.
Labrinidis and Roussopoulos balance performance against data freshness in web
database servers by scheduling view updates~\cite{labrinidis2004}; Guo et al.
formalize relaxed currency and consistency so queries can state how stale an
answer may be~\cite{guo2004}; R\"ohm et al. route queries by freshness across
replicated OLAP components~\cite{rohm2002}; and Olston and Widom study
best-effort cache synchronization under a precision/freshness
tradeoff~\cite{olston2002}. Most directly, Cho and Garcia-Molina analyze how to
schedule refreshes of a local copy to maximize freshness against sources that
change at heterogeneous rates, and obtain the celebrated, counterintuitive
result that for a pure-freshness objective uniform refresh is often optimal and
one may even refresh slow-changing items \emph{more}~\cite{cho2000}. Our
objective is a weight- and sensitivity-weighted stale-exposure functional
rather than the fraction of time up to date, and under it the optimal
allocation moves in the intuitive direction (refresh high-risk partitions
more), with a clean square-root law. The two results are consistent and we
contrast them directly in Section~\ref{sec:discussion}: both turn on the
objective, and both identify homogeneity as the point where uniform is optimal.

\paragraph{Caching, TTL, and Age of Information.}
Time-to-live expiration and web-cache freshness policies are the operational
ancestors of a single-age staleness budget~\cite{cao1997}: a TTL is one age
threshold. The equivalence limit says a global weighted staleness score
collapses back to a (recalibrated) TTL. Age-of-Information provides a vocabulary
and optimization results for the age of information observed by a
consumer~\cite{yates2021}; our exposure functional is in that spirit, and our
question is when an age-derived control variable carries scheduling information
beyond age itself.

\paragraph{Model refresh as derived-state maintenance.}
Production ML platforms treat data quality and train/serve consistency as
first-class operational concerns~\cite{sculley2015,baylor2017,breck2019}, and
concept-drift research studies when predictive distributions
change~\cite{gama2014}. This paper is orthogonal to drift detection: it does not
ask \emph{when} the world changed but how to allocate a fixed recompute budget
across heterogeneous partitions of a derived artifact, whether that artifact is
a relational aggregate or a trained model.

\section{System Model}
\label{sec:model}
\subsection{Partitioned Derived State and Stale Exposure}
Let the derived artifact be partitioned into $n$ partitions $j=1,\dots,n$---for
a materialized aggregate, the groups of a \texttt{GROUP BY}; for a model, data
segments or per-segment sub-models. Partition $j$ has importance weight
$\pi_j \ge 0$, sensitivity $s_j \ge 0$, and a Poisson base-change rate
$\lambda_j \ge 0$ (the rate at which rows in its base data are inserted or
updated); write $w_j = \pi_j s_j$ for its combined weight. Interpreting
$1 - e^{-\lambda_j a}$ as the probability that partition $j$ has an
unincorporated base change at age $a$ since its last recompute---a standard
Poisson-exposure model---the long-run average weighted stale exposure of a
partition recomputed periodically with interval $T_j$ is
\begin{equation}
\label{eq:Ej}
E_j(T_j) \;=\; \frac{w_j}{T_j}\int_0^{T_j}\!\bigl(1-e^{-\lambda_j a}\bigr)\,da
\;=\; w_j\left[\,1 - \frac{1-e^{-\lambda_j T_j}}{\lambda_j T_j}\,\right].
\end{equation}
The system cost is the total weighted exposure $C = \sum_j E_j(T_j)$, reported
normalized by $\sum_j w_j$.

\subsection{Recompute Budget and the Separable-Cost Assumption}
\label{sec:cost}
Each recompute of a partition consumes resources; we charge one unit per
partition-recompute and constrain the total \emph{recompute rate}
\begin{equation}
\label{eq:budget}
R \;=\; \sum_{j=1}^n \frac{1}{T_j}
\quad\text{(partition-recomputes per unit time).}
\end{equation}
This budget assumes recompute cost is \emph{separable across partitions}: the
cost of refreshing a subset is the sum of the subset's costs. Partition-wise
materialized aggregates (recompute one group's view), per-tenant or per-region
views, and per-segment sub-models all satisfy it. A monolithic artifact
refreshed only by full global recompute does not: refreshing any part costs a
full recompute, and differential refresh offers no saving---which, by
Proposition~\ref{prop:equiv}, is exactly the single-age regime where a staleness
score is already a timer. The constructive results hold precisely when partial
recompute is meaningful.

\section{The Single-Age Equivalence Limit}
\label{sec:limit}
A single global recompute sets one common cutoff for every partition, so after
it all partitions share the same age $a(t)$, and any staleness score is a
function of that one scalar.

\begin{proposition}[Single-age equivalence]
\label{prop:equiv}
Let a refresh trigger fire when $G(a(t)) \ge \theta$, where $a(t)$ is a single
shared age and $G$ is static and strictly monotone increasing on the attainable
age interval. Then for every threshold $\theta$ in the attainable range of $G$
there is a unique age $\tau_\theta = G^{-1}(\theta)$ such that the trigger fires
if and only if $a(t) \ge \tau_\theta$. The trigger therefore induces the same
schedule as a calibrated uniform timer of period $\tau_\theta$.
\end{proposition}
\begin{proof}
$G$ strictly increasing admits a well-defined inverse on its attainable range,
and $G(a) \ge \theta \iff a \ge G^{-1}(\theta)$. Under global recompute $a(t)$
resets to $0$ at each refresh and otherwise increases at unit rate, so the first
crossing of $a \ge \tau_\theta$ recurs with period $\tau_\theta$.
\end{proof}

\begin{corollary}[A global staleness budget is a timer]
\label{cor:wse}
The normalized weighted stale exposure $F(a) = \sum_j w_j(1-e^{-\lambda_j a})/\sum_j w_j$
evaluated at a shared age is static and strictly monotone in $a$, since
$F'(a) = \sum_j w_j\lambda_j e^{-\lambda_j a}/\sum_j w_j > 0$ whenever some
$w_j\lambda_j > 0$. By Proposition~\ref{prop:equiv}, thresholding it at $\beta$
is a uniform timer of period $\tau_\beta = F^{-1}(\beta)$.
\end{corollary}

Figure~\ref{fig:equiv} is the geometric content: the strictly monotone $F$
carries each budget $\beta$ to a unique age $\tau_\beta$, so a horizontal
threshold and a vertical age cut are the same line on two axes. To beat the
timer one must leave the single-age regime---which means refreshing partitions
differentially.

\begin{figure}[t]
\centering
\includegraphics[width=0.62\linewidth]{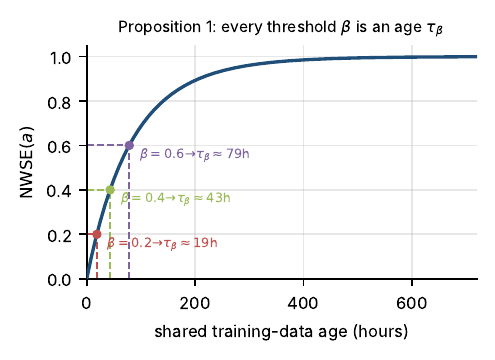}
\caption{The equivalence limit (Proposition~\ref{prop:equiv}). $\mathrm{NWSE}(a)$
for a heterogeneous six-partition set is strictly monotone in the shared age, so
each budget $\beta$ maps through $F^{-1}$ to a unique age $\tau_\beta$:
thresholding the score is thresholding age in different units.}
\label{fig:equiv}
\end{figure}

\section{Optimal Differential Refresh}
\label{sec:policy}
\subsection{The Allocation Problem}
Allow each partition its own interval $T_j$ (rate $x_j = 1/T_j$). The design
problem is to minimize total weighted exposure subject to the recompute budget:
\begin{equation}
\label{eq:opt}
\min_{T_1,\dots,T_n > 0}\ \sum_{j} E_j(T_j)
\qquad\text{s.t.}\qquad \sum_j \frac{1}{T_j} = R .
\end{equation}
The uniform timer is the feasible point $T_j = n/R$; it is the policy the limit
collapses to.

\begin{proposition}[Optimal differential allocation]
\label{prop:alloc}
Problem~\eqref{eq:opt} has optimal intervals $T_j^\star = u_j^\star/\lambda_j$,
where each $u_j^\star > 0$ solves the stationarity condition
\begin{equation}
\label{eq:kkt}
\frac{w_j}{\lambda_j}\bigl(1 - e^{-u_j}(1+u_j)\bigr) \;=\; \mu
\qquad (j = 1,\dots,n),
\end{equation}
with a single multiplier $\mu > 0$ fixed by $\sum_j \lambda_j/u_j^\star = R$. In
the frequent-refresh regime $u_j \to 0$ this reduces to the closed form
\begin{equation}
\label{eq:sqrt}
x_j^\star = \frac{1}{T_j^\star} = R\,\frac{\sqrt{w_j\lambda_j}}{\sum_k \sqrt{w_k\lambda_k}},
\qquad\text{i.e.}\quad x_j^\star \propto \sqrt{w_j\lambda_j}.
\end{equation}
\end{proposition}
\begin{proof}
The Lagrangian of~\eqref{eq:opt} is $\sum_j E_j(T_j) + \mu(\sum_j 1/T_j - R)$.
Setting $\partial/\partial T_j = 0$ gives $E_j'(T_j) = \mu/T_j^2$. With
$u=\lambda_j T_j$, differentiating \eqref{eq:Ej} gives
$E_j'(T_j) = w_j\lambda_j(1 - e^{-u}(1+u))/u^2 > 0$, so
$T_j^2 E_j'(T_j) = (w_j/\lambda_j)(1-e^{-u_j}(1+u_j)) = \mu$, which is
\eqref{eq:kkt}; the left side is increasing in $u_j$, so $u_j^\star$ is unique
given $\mu$, and $\mu$ is set by the budget. As $u\to 0$,
$1 - e^{-u}(1+u) = u^2/2 + O(u^3)$, so \eqref{eq:kkt} becomes
$w_j\lambda_j T_j^2/2 = \mu$, giving $T_j^\star = \sqrt{2\mu/(w_j\lambda_j)}$ and
$x_j^\star \propto \sqrt{w_j\lambda_j}$; normalizing to the budget yields
\eqref{eq:sqrt}.
\end{proof}

A partition's recompute rate should scale with the square root of its risk
$w_j\lambda_j$, not with the risk itself. Figure~\ref{fig:alloc} confirms the
exact numerical optimum follows \eqref{eq:sqrt} tightly in the frequent regime
(log--log slope $1/2$, correlation $0.998$). Outside it the closed form degrades
and the numerical solution of~\eqref{eq:kkt} should be used
(Section~\ref{sec:results}).

\begin{figure}[t]
\centering
\includegraphics[width=0.62\linewidth]{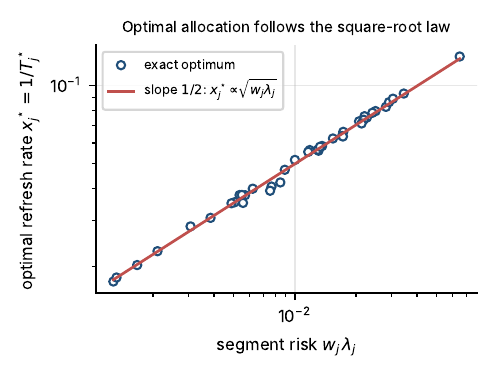}
\caption{The optimal allocation (Proposition~\ref{prop:alloc}). In the
frequent-refresh regime the exact numerical optimum obeys the closed-form
square-root law $x_j^\star \propto \sqrt{w_j\lambda_j}$ (log--log slope $1/2$);
40-partition moderate-heterogeneity profile.}
\label{fig:alloc}
\end{figure}

\subsection{Dominance and the Price of Uniformity}
\begin{proposition}[Dominance with a closed-form gap]
\label{prop:dominance}
The optimal differential policy never costs more than the uniform timer at the
same budget. In the frequent-refresh regime the relative saving is
\begin{equation}
\label{eq:gap}
1 - \frac{C^\star}{C_{\mathrm{unif}}}
\;=\; 1 - \frac{\bigl(\sum_j \sqrt{w_j\lambda_j}\bigr)^2}{n\,\sum_j w_j\lambda_j},
\end{equation}
which lies in $[0,1)$, equals $0$ iff $w_j\lambda_j$ is the same for all $j$, and
increases with the dispersion of $\sqrt{w_j\lambda_j}$.
\end{proposition}
\begin{proof}
Optimality gives $C^\star \le C_{\mathrm{unif}}$ since the uniform allocation is
feasible. In the frequent regime $E_j(T_j) \approx w_j\lambda_j T_j/2$, so with
$a_j := w_j\lambda_j$ the uniform cost is $C_{\mathrm{unif}} = (n/2R)\sum_j a_j$
and, substituting \eqref{eq:sqrt}, $C^\star = (1/2R)(\sum_j \sqrt{a_j})^2$. Their
ratio $(\sum_j\sqrt{a_j})^2/(n\sum_j a_j) \le 1$ by Cauchy--Schwarz, with
equality iff all $a_j$ are equal.
\end{proof}

Equation~\eqref{eq:gap} is a ``price of uniformity'': the fractional cost a
uniform timer pays for ignoring partition risk is a normalized dispersion of
$\sqrt{w_j\lambda_j}$, zero for homogeneous workloads (recovering the
equivalence limit) and growing with heterogeneity.

\subsection{An Implementable Policy, and a Caution}
\label{sec:impl}
Proposition~\ref{prop:alloc} yields a directly implementable policy: estimate
$w_j,\lambda_j$, solve~\eqref{eq:kkt} (or use \eqref{eq:sqrt}) for the target
intervals $T_j^\star$, and recompute each partition on its own timer. A tempting
simpler rule---recompute partition $j$ when its exposure
$w_j(1-e^{-\lambda_j a_j})$ crosses a common threshold---needs no allocation
step but is \emph{not} optimal: its intervals scale as $1/(w_j\lambda_j)$ rather
than $1/\sqrt{w_j\lambda_j}$, over-concentrating recomputes on high-risk
partitions. Sections~\ref{sec:results}--\ref{sec:realsys} show this naive policy
is erratic and does not reliably beat the uniform timer, so the
allocation---not merely going differential---is the contribution.

\section{Analytic and Simulation Evaluation}
\label{sec:results}
\paragraph{Setup.} We evaluate on synthetic partitioned workloads with $n=8$
partitions and four heterogeneity profiles (\emph{low}, \emph{moderate},
\emph{high}, \emph{extreme}) controlling the spread of $\lambda_j$ and $w_j$.
Analytic results solve~\eqref{eq:opt} numerically (SLSQP) and in closed
form~\eqref{eq:sqrt} over $50$ random profiles per cell, with the uniform timer
as baseline at the identical budget. The discrete-event simulation runs a
$4320$-hour horizon on a $6$-hour grid, drawing actual Poisson change events;
policies are compared on the same event stream at matched total recompute count,
$50$ seeds. The real-system study is deferred to Section~\ref{sec:realsys}.

\subsection{The Analytic Win Grows with Heterogeneity}
Table~\ref{tab:win} reports the relative saving of the exact optimal allocation
over the uniform timer. It is essentially zero for near-homogeneous workloads,
where Proposition~\ref{prop:dominance} predicts no gain, and rises monotonically
with heterogeneity and refresh frequency, reaching $29.5\%$ for extreme
heterogeneity. The optimizer never produced a point worse than uniform across
all $600$ profiles, consistent with the dominance guarantee.

\begin{table}[t]
\centering
\caption{Analytic cost saving of optimal differential refresh over the uniform
timer (mean over 50 profiles, \%).}
\label{tab:win}
\begin{tabular}{lccc}
\toprule
Heterogeneity & Sparse & Moderate & Frequent \\
\midrule
Low      & 0.2 & 0.2 & 0.3 \\
Moderate & 7.5 & 7.1 & 8.4 \\
High     & 11.4 & 12.9 & 14.0 \\
Extreme  & 21.0 & 24.1 & 29.5 \\
\bottomrule
\end{tabular}
\end{table}

\subsection{The Closed Form Is Exact in the Frequent Regime}
The square-root law~\eqref{eq:sqrt} is an asymptote. In the frequent regime it
is highly accurate: its cost exceeds the exact optimum by at most $0.35\%$ at
moderate heterogeneity, and optimal and closed-form rates correlate at $0.98$
(log scale). In the sparse regime it degrades (cost excess up to $22\%$ at
extreme heterogeneity, where partitions saturate, $\lambda_j T_j \gg 1$), so the
numerical solution of~\eqref{eq:kkt} should be used there. This is the expected
boundary of a frequent-regime closed form.

\subsection{The Policy Wins in Simulation at Matched Budget}
Table~\ref{tab:sim} and Figure~\ref{fig:sim} give the simulation result. The
optimal differential policy lowers realized weighted stale exposure by
$7.7$--$28.8\%$ over the uniform timer at matched recompute count, winning on
$86$--$100\%$ of seeds. The naive exposure-threshold policy hovers near
break-even and is frequently worse than uniform (down to $-8.7\%$): going
differential helps only when the budget is allocated by the square-root law.

\begin{table}[t]
\centering
\caption{Realized stale-exposure saving over the uniform timer in the
discrete-event simulation, at matched recompute count (50 seeds). $K$ is the
uniform timer's recompute count; ``Opt.'' is optimal differential; ``Naive'' is
the exposure threshold; ``seeds'' is the fraction the optimal policy wins.}
\label{tab:sim}
\begin{tabular}{llccc}
\toprule
Heterogeneity & $K$ & Opt.\ (\%) & Naive (\%) & Opt.\ seeds (\%) \\
\midrule
Moderate & 10 & \phantom{0}8.6 & \phantom{$-$}0.2 & \phantom{0}98 \\
Moderate & 20 & \phantom{0}7.7 & $-4.3$ & \phantom{0}86 \\
Moderate & 40 & \phantom{0}8.6 & $-8.1$ & \phantom{0}94 \\
High     & 10 & 11.0 & \phantom{$-$}4.1 & 100 \\
High     & 20 & 12.7 & \phantom{$-$}3.0 & 100 \\
High     & 40 & 16.0 & \phantom{$-$}2.5 & 100 \\
Extreme  & 10 & 15.9 & $-5.9$ & \phantom{0}96 \\
Extreme  & 20 & 19.6 & $-8.7$ & \phantom{0}98 \\
Extreme  & 40 & 28.8 & $-6.4$ & \phantom{0}98 \\
\bottomrule
\end{tabular}
\end{table}

\begin{figure}[t]
\centering
\includegraphics[width=0.66\linewidth]{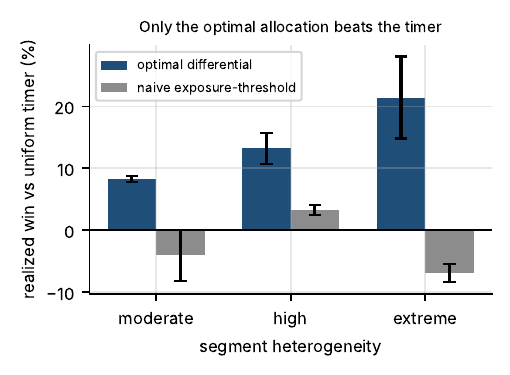}
\caption{Simulated saving over the uniform timer at matched recompute count
(mean over the three budgets of Table~\ref{tab:sim}, $\pm$ s.d.). Optimal
differential wins and grows with heterogeneity; the naive exposure threshold is
near break-even or worse.}
\label{fig:sim}
\end{figure}

\subsection{Robust to Rate-Estimation Noise}
Table~\ref{tab:noise} recomputes the allocation from multiplicative log-normal
noisy $\lambda_j$ while charging cost against the true rates. The win is stable:
even at $50\%$ noise it retains $5.5\%$ of $7.0\%$ (moderate), $11.3\%$ of
$13.0\%$ (high), and $23.2\%$ of $24.3\%$ (extreme). Because the allocation
depends on $\sqrt{\lambda_j}$, rate errors are damped.

\begin{table}[t]
\centering
\caption{Analytic win (\%) under multiplicative $\lambda$-estimation noise at the
moderate budget (50 profiles, mean).}
\label{tab:noise}
\begin{tabular}{lcccc}
\toprule
Heterogeneity & noise 0.00 & 0.10 & 0.25 & 0.50 \\
\midrule
Moderate & 7.0 & 7.0 & 6.6 & 5.5 \\
High     & 13.0 & 12.9 & 12.6 & 11.3 \\
Extreme  & 24.3 & 24.3 & 24.0 & 23.2 \\
\bottomrule
\end{tabular}
\end{table}

\section{Real-System Evaluation on Materialized Views}
\label{sec:realsys}
Analytic and discrete-event results idealize timing and recompute cost. To test
the policy where scheduling jitter, real SQL execution, and I/O replace the
idealized clock, we implement it on a real relational engine.

\paragraph{System.} Each partition is a materialized aggregate
(\texttt{GROUP BY}) view over a base table in SQLite~3.45. A source applies real
\texttt{INSERT}s over wall-clock time at a per-partition Poisson rate $\lambda_j$
(base data changing). A scheduler keeps each view fresh by actually recomputing
it with SQL (\texttt{DELETE} + \texttt{INSERT INTO mv\_j SELECT \ldots GROUP BY})
and writing it back; this is derived-state maintenance on a real engine, not a
discrete-event abstraction---real SQL execution, real I/O, real timing. We
compare three policies at a matched number of equally-charged view recomputes:
\textsc{uniform} (every view on one global period---the collapsed single-age
timer of Proposition~\ref{prop:equiv}), \textsc{differential} (each view on its
own $T_j^\star$ from the square-root allocation), and \textsc{naive} (the
per-view exposure threshold). The realized cost is time-integrated weighted
staleness, $\sum_j w_j$ times the fraction of wall-clock time view $j$ holds
$\ge 1$ unincorporated row---the real-system realization of $E_j$. We report
$n=8$ partitions, $8$ seeds per heterogeneity profile, in the frequent-refresh
regime where the square-root law applies (environment: SQLite~3.45.1,
Python~3.13, Linux x86-64).

\paragraph{Results.} Table~\ref{tab:realsys} and Figure~\ref{fig:realsql} give
the result. On the real engine the optimal differential policy lowers realized
weighted staleness by $14.2\%$ (moderate) and $15.0\%$ (high heterogeneity),
\emph{winning on all $8$ seeds in both regimes}, at matched recompute count. The
naive exposure threshold is erratic---mean $-1.4\%$ (moderate) and $+0.2\%$
(high), with large variance---and the optimal policy beats it on $7$ of $8$
seeds in each regime. The simulation's central finding therefore survives
contact with a real relational engine: differential refresh helps, but only
when the budget is allocated by the square-root law.

\begin{table}[t]
\centering
\caption{Real-system validation on SQLite materialized \texttt{GROUP BY} views
(8 seeds/profile, matched recompute count). ``Diff.''\ and ``Naive''\ are the
staleness reduction vs.\ the uniform timer; ``wins''\ counts seeds the policy
beats uniform; recompute time is total SQL recompute wall-time per run.}
\label{tab:realsys}
\begin{tabular}{lccccc}
\toprule
Heterogeneity & Diff.\ (\%) & Diff.\ wins & Naive (\%) & Uniform (ms) & Diff.\ (ms) \\
\midrule
Moderate & $14.2 \pm 5.8$ & 8/8 & $-1.4 \pm 16.4$ & 343 & 589 \\
High     & $15.0 \pm 6.8$ & 8/8 & $+0.2 \pm 13.9$ & 388 & 633 \\
\bottomrule
\end{tabular}
\end{table}

\begin{figure}[t]
\centering
\includegraphics[width=0.66\linewidth]{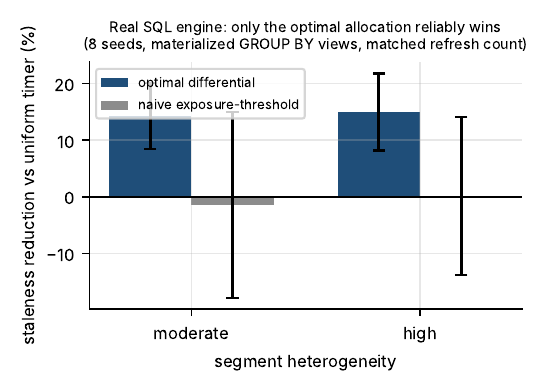}
\caption{Real relational engine (SQLite materialized \texttt{GROUP BY} views,
8 seeds, matched recompute count). The optimal differential policy reliably
reduces weighted staleness ($+14$--$15\%$, winning every seed); the naive
exposure threshold straddles zero. Error bars are $\pm$ s.d.}
\label{fig:realsql}
\end{figure}

\paragraph{A real-system nuance.} At matched recompute \emph{count}, the
differential policy spends more total SQL recompute time (Table~\ref{tab:realsys},
last two columns: $\approx\!590$--$630$\,ms vs.\ $\approx\!340$--$390$\,ms)
because it recomputes high-change partitions more often, and those views are
heavier to rebuild. The matched-count budget is the right model when recompute
cost is uniform across partitions; when it is not, the extension of
Section~\ref{sec:limits} (per-partition cost $c_j$, giving
$x_j^\star \propto \sqrt{w_j\lambda_j/c_j}$) re-budgets on CPU rather than count.
We report this explicitly because it is visible only on a real engine.

\section{Discussion}
\label{sec:discussion}
\paragraph{Why the square root, and the contrast with Cho--Garcia-Molina.}
The intuitive guess is to recompute each partition in proportion to its risk
$w_j\lambda_j$; the optimum scales with its \emph{square root}. The extra
concavity comes from the budget being a sum of \emph{rates} $1/T_j$ while cost is
(to first order) linear in the intervals $T_j$, so equalizing marginal cost per
unit rate yields the square-root rule---and the naive exposure threshold, whose
intervals scale as $1/(w_j\lambda_j)$, over-serves high-risk partitions and
loses. Cho and Garcia-Molina found that for a pure-freshness objective uniform
refresh is often optimal and non-uniformity can invert intuition~\cite{cho2000};
our weighted-exposure objective makes non-uniformity help in the intuitive
direction. Both results turn on the objective, and both identify homogeneity as
where uniform is optimal.

\paragraph{When differential refresh does not help.}
If the workload is homogeneous ($w_j\lambda_j$ nearly constant), the price of
uniformity~\eqref{eq:gap} is near zero. If recompute cost is not separable---a
monolithic full recompute---the budget~\eqref{eq:budget} does not apply and there
is nothing to allocate; by Proposition~\ref{prop:equiv} that is exactly the
single-age regime where a staleness score is already a timer. And if rate
estimates are uninformative the allocation reverts toward uniform, though the
noise study shows this degradation is slow.

\section{Limitations and Future Work}
\label{sec:limits}
The real-system study is single-node (one relational engine) and uses a
synthetic Poisson change stream; validation on a distributed/replicated cluster
and on production update traces is future work, as is extension to
\texttt{JOIN}-bearing views whose recompute cost depends on more than one base
table. The exposure model is a Poisson single-change abstraction; richer change
processes and non-exponential staleness-to-error couplings are future work,
though the allocation framework is agnostic to the exact $E_j$ as long as it is
convex and increasing in $T_j$. Heterogeneous per-partition recompute costs
$c_j$ are a direct extension: replacing $\sum 1/T_j$ by $\sum c_j/T_j$ rescales
the law to $x_j^\star \propto \sqrt{w_j\lambda_j/c_j}$, which, as
Section~\ref{sec:realsys} notes, is the right budget when recompute cost varies
across partitions. Online estimation of $\lambda_j$ under drift and joint
optimization of refresh with partition granularity are natural next steps.

\section{Conclusion}
Refreshing derived state---a materialized aggregate or a trained model---is a
scheduling problem, and a freshness budget is only as powerful as the action it
drives. Under a single global recompute, a weighted staleness score of one
shared age is a uniform timer in disguise (Proposition~\ref{prop:equiv}). The
leverage is a better action: recompute partitions differentially and allocate
the recompute budget by the square root of partition risk
(Proposition~\ref{prop:alloc}), which never costs more than the uniform timer
and saves a closed-form Cauchy--Schwarz dispersion that grows with heterogeneity
(Proposition~\ref{prop:dominance}). On a real relational engine with
materialized \texttt{GROUP BY} views, the policy cuts realized weighted staleness
by $14$--$15\%$ at matched recompute count, winning every seed, while a naive
exposure threshold does not. The design rule: a global staleness score is a
timer; buy the win with differential refresh, not a more intricate score.

\section*{Declarations}
\noindent\textbf{Funding.} The author received no funding for this work.

\noindent\textbf{Competing interests.} The author declares no competing
interests.

\noindent\textbf{Data and code availability.} The discrete-event simulator, the
real-system (SQLite) materialized-view harness, and all scripts that generate
every table and figure are provided as supplementary material; all reported
numbers are reproducible from that code with the random seeds stated in the text.

\section*{Declaration of generative AI and AI-assisted technologies in the manuscript preparation process}
\noindent During the preparation of this work, the author used Anthropic Claude
(Claude Opus 4.8, \url{https://www.anthropic.com}) to support the mathematical
derivations, to assist in implementing the simulation and real-system evaluation
code, and to draft and edit the text. The author designed the study, verified all
theoretical results and proofs, independently checked every reported number
against the code, reviewed and edited the output as needed, and takes full
responsibility for the content of the published article.

\end{document}